\documentclass{article}
\usepackage{inputenc}

\usepackage{amsmath,amsbsy,amssymb,amsthm,bbm,bm}
\usepackage{mathtools}

\usepackage{multirow,multicol,tabularx,booktabs}
\usepackage{placeins,color,url}
\usepackage{bbm,bm}
\usepackage[ruled]{algorithm2e}
\usepackage[shortlabels]{enumitem}
\usepackage{wrapfig}
\usepackage{relsize}
\usepackage{subfigure}
\usepackage{caption}
\usepackage[draft]{fixme}
\usepackage{authblk}
\fxsetup{layout=margin}
\fxusetheme{color}

\usepackage{mathrsfs}
\usepackage{graphicx}
\usepackage[top=2.5cm,bottom=2.5cm,right=2.5cm,left=2.5cm]{geometry}
\usepackage[hidelinks]{hyperref}

\allowdisplaybreaks

\newtheorem{conj}{Conjecture}[section]
\newtheorem{theorem}[conj]{\bf Theorem}
\newtheorem{definition}[conj]{\bf Definition}
\newtheorem{corollary}[conj]{\bf Corollary}

\newtheorem{lemma}[conj]{\bf Lemma}

\newtheorem{innercustomgeneric}{\customgenericname}
\providecommand{\customgenericname}{}
\newcommand{\newcustomtheorem}[2]{%
  \newenvironment{#1}[1]
  {%
   \renewcommand\customgenericname{#2}%
   \renewcommand\theinnercustomgeneric{##1}%
   \innercustomgeneric
  }
  {\endinnercustomgeneric}
}

\newcustomtheorem{customthm}{Assumption}

\newcommand{\resolution}{r}

\newcommand{\VC}{\operatorname{VC}}

\newcommand{\seven}{\bgroup% in case box 0 being used
  \sbox0{7}\usebox0\llap{\rule[.5\ht0]{.4\wd0}{.05\ht0}\rule{.24\wd0}{0pt}}
\egroup}

\def\to{\rightarrow}

\def\R{{\mathbb R}}

\def\EE{ {\rm I} \kern-.15em {\rm E} }
\def\PP{ {\rm I} \kern-.15em {\rm P} }

\def\1{\mathbbm{1}}

\usepackage{import}
\usepackage{xifthen}
\usepackage{pdfpages}
\usepackage{transparent}

\usepackage[round]{natbib}

\title{On the VC dimension of deep group convolutional neural networks}
\author[1,2]{Anna Sepliarskaia}
\author[1]{Sophie Langer}
\author[1]{Johannes Schmidt-Hieber}

\affil[1]{Department of Applied Mathematics, University of Twente}
\affil[2]{Booking.com}
\begin{document}

\listoffixmes
\clearpage

\maketitle
\begin{abstract}
    We study the generalization capabilities of Group Convolutional Neural Networks (GCNNs) with ReLU activation function by deriving upper and lower bounds for their Vapnik-Chervonenkis (VC) dimension. Specifically, we analyze how factors such as the number of layers, weights, and input dimension affect the VC dimension. We further compare the derived bounds to those known for other types of neural networks. Our findings extend previous results on the VC dimension of continuous GCNNs with two layers, thereby providing new insights into the generalization properties of GCNNs, particularly regarding the dependence on the input resolution of the data.
\end{abstract}

\section{Introduction}
Convolutional Neural Networks (CNNs) have revolutionized the field of computer vision, achieving remarkable success in tasks such as image classification \citep{krizhevsky2012imagenet}, object detection~\citep{ren2016faster}, and segmentation~\citep{long2015fully}.  Their effectiveness can be partly attributed to their translation invariant architecture, enabling CNNs to recognize objects regardless of their position in an image. However, while CNNs are effective at capturing translation symmetries, there has been a growing interest in incorporating additional structure into neural networks to handle a wider range of transformations. These architectures aim to combine the flexibility of learning with the robustness of structure-preserving features (see, e.g., \citep{hinton2011, Lee2015}).
\\
\\
GCNNs were first introduced by \cite{pmlr-v48-cohenc16} to improve statistical efficiency and enhance geometric reasoning. Since then equivariant network structures have evolved to support equivariance on Euclidean groups \citep{bekkers2018roto, bekkers2019b, weiler2018learning}, compact groups \citep{kondor2018generalization} and Riemannian manifolds \citep{weiler2021coordinate}. More recent architectures have even been generalized beyond other types of symmetry groups \citep{zhdanov2024clifford, dehmamy2021automatic, MR4541121}
\\
\\
These advancements have significantly broadened the applicability of CNNs to more complex tasks including fluid dynamics ~\citep{wang2020incorporating, vinuesa2022enhancing}, electrodynamics \citep{zhdanov2024clifford}, medical image segmentation \citep{bekkers2018roto}, partial differential equation (PDE) solvers \citep{brandstetter2022lie}, and video tracking \citep{sosnovik2021scale}.
A notable example of the success of equivariant neural networks is the AlphaFold algorithm, which achieves high accuracy in protein structure prediction through a novel roto-translation equivariant attention architecture of the neural network \citep{jumper2021highly, fuchs2020se}.

Equivariant neural networks are expected to have smaller sample complexity if compared to neural networks without built-in symmetries \citep{bietti2021sample, sannai2021improved, elesedy2022group, shao2022theory}. Given the symmetries are present in the data, equivariant networks should therefore achieve comparable performance with fewer examples.

To investigate this hypothesis, we study the generalization capabilities of a GCNN class that is invariant to group actions. 
We compute upper and lower bounds for the Vapnik-Chervonenkis (VC) dimension of this class and compare them with the VC dimension of deep feedforward neural networks (DNNs).

Previously, the VC dimension has been used to analyze the generalization properties of various network architectures, including vanilla DNNs \citep{bartlett2019nearly, anthony1999neural} and CNNs \citep{kohlerwalter2023}. However, to the best of our knowledge, the only result on the VC dimension of GCNNs pertains to two-layer GCNNs with continuous groups, which were shown to have an infinite VC dimension \citep{petersen2024vc}.
Our work differs from earlier research by analyzing the VC dimension of GCNNs in relation to the number of layers, weights, and input dimensions of the neural network, without restricting the architecture to just two layers (as in \citep{petersen2024vc}) or DNNs (as in \citep{bartlett2019nearly}).

\section{GCNNs}

\subsection{Group theoretic concepts}
GCNNs capture symmetries in data via group theory. A \textit{group} $G$ is a set equipped with an operation $\circ$ such that $h,g \in G$ implies $h\circ g \in G;$ there exists an identity element $e\in G$ such that  \( e \circ g = g \circ e = g \) for all \( g \in G \); there exists an inverse element \( g^{-1} \in G \) such that \( g^{-1}\circ g = g \circ g^{-1} = e \); and for any \( g, h, i \in G \), we have \( (g \circ h) \circ i = g \circ (h \circ i) \).

A \textbf{group action} describes how a group interacts with another set. More specifically, an \emph{action} of $G$ on a set $\mathcal{X}$ is a map $\circ: G \times \mathcal{X} \rightarrow \mathcal{X}$ such that for all $g_1, g_2\in G$ and for all $x\in \mathcal{X}$
\begin{align*}
    (g_1 g_2) \circ x = g_1\circ(g_2\circ x).   
\end{align*}

On the functions $\{ f : G \to \mathbb{R} \},$ the \emph{left regular representation} of a group \( G \) is the action 
\begin{align}
    (g \circ f)(g') = f(g^{-1} g'), \quad \text{for any} \ g, g' \in G.    
    \label{eq.left_reg}
\end{align}

For a so-called kernel function \( \mathcal{K}: G \rightarrow \mathbb{R} \) and a function \( f: G \rightarrow \mathbb{R} \), the group convolution of \( \mathcal{K} \) with \( f \) at an element \( g \in G \) is defined as
\begin{equation}
\label{Eq:convolution}
(\mathcal{K} \star f)(g) := \int_{G} \mathcal{K}(g^{-1} g') \cdot f(g') \, d\mu(g'),
\end{equation}
with \( \mu \) being the Haar measure \cite[Theorem 8.1.2]{procesi2007lie}. Note that, the Haar measure is a unique left-invariant measure that exists for all locally compact groups (see~\cite[Definition 1.18]{stroppel2006locally}). For the integral in \eqref{Eq:convolution} to be well-defined, we assume that both, the signal \( f \) and the kernel \( \mathcal{K} \), are measurable and bounded. 

The group convolution computes a weighted average over the group elements. It is equivariant with respect to the left regular representation of \( G \) defined in \eqref{eq.left_reg}, that is, \(g \circ (\mathcal{K} \star f) = (\mathcal{K} \star (g \circ f))
\) for any group element \( g \).

In the case where \( G \) is a finite group, the group convolution becomes the sum
\begin{align}
\label{convolution_discrete}
(\mathcal{K} \star f)(g) = \frac{1}{|G|} \sum_{g' \in G} \mathcal{K}( g^{-1} g') \cdot f(g').
\end{align}

In practice, continuous groups such as rotations or translations are often discretized, and computations are performed on finite grids. This discretization process involves approximating the group \( G \) by a finite subset 
\[
G^r = \{ g_1, g_2, \dots, g_r \}.
\]
We refer to the cardinality \( r \) as the resolution of the discretization. 

For example, for the group formed by all continuous rotations, the discretization selects a finite number of rotation angles. For the translation group, discretization consists of a finite number of shifts. 

After discretization, the group convolution operates on this finite set. Ignoring the reweighting it is then given by
\begin{align}
\label{convolution_discrete2}
(\mathcal{K} * f)(g) := \sum_{j=1}^r \mathcal{K}( g^{-1} g_j) \cdot f(g_j).
\end{align}
This operation is called \textit{G-correlation} \citep{cohen2016group}. Note that $(\mathcal{K} * g)=r \cdot (\mathcal{K} \star g)$. If $\mathcal{K}=\mathbf{1}(\cdot=e)$ with $e$ the identity element of the group, then, on $G^r,$ $K * f= f$.
The $G$-correlation heavily depends on the discretization of the group and can differ significantly from the integral version \eqref{Eq:convolution}. However, both definitions become approximately the same (up to rescaling) if the elements $g_j$ are drawn from the Haar measure. In the case of GCNNs, the discretization is determined by the structure of the data.

\subsection{GCNN architecture}
\label{sec.architecture}

To parametrize the kernel function, let 
\begin{equation}
\label{kernel_space}
\texttt{K}_s : G \rightarrow \mathbb{R},\quad  s = 1, \dots, k    
\end{equation}
be a set of basis functions. The kernel functions \( \mathcal{K}_{\mathbf{w}} \) are then expressed as linear combinations of these basis functions, that is,
\begin{align}
\mathcal{K}_{\mathbf{w}} = \sum_{s=1}^{k} w_s \texttt{K}_s,
 \label{eq.bcusdyb}
\end{align}
where \( \mathbf{w} = (w_1, \dots, w_k) \) is the vector of trainable parameters.

Given an activation function $\sigma: \R\to \R,$ a group convolutional unit (GCNN unit) or G-convolution takes an input function on the discretized group \(f=(f_1,\ldots, f_{m_0})^T: G^r \rightarrow \mathbb{R}^{m_0} \) as input and outputs another function on the discretized group \( h: G^r \rightarrow \mathbb{R} \). The output $h$ is  
\begin{align}
\label{cu_GCNN2}
  h=\sigma \bigg(\sum_{i=1}^{m_0}   \mathcal{K}_{\mathbf{w}_i} * f_i - b\bigg),
\end{align}
with the group convolution operation $*$ defined in \eqref{convolution_discrete2}. The weight vectors $\mathbf{w}_1,\ldots,\mathbf{w}_m$ and the bias $b$ are parameters that are learned from the data. In line with the common terminology, we refer to the output as feature map. 

A GCNN layer (also called $G$-convolutional layer or $G$-conv layer) computes several GCNN units in parallel, using the same input functions but applying different kernel functions (also known as filters), each with potentially different parameters. Specifically, a GCNN layer with $M$ units and input function \( f = (f_1, \dots, f_{m_0})^T: G^r \rightarrow \mathbb{R}^{m_0},  \) computes the following $M$ functions  
\begin{align}
\label{cu_GCNN3}
  h_j=\sigma \bigg(\sum_{i=1}^{m_0}   \mathcal{K}_{\mathbf{w}_{ij}} * f_i - b_j\bigg), \quad j=1,\ldots, M,
\end{align}
where $\mathcal{K}_{\mathbf{w}_{ij}}$ connects the $i$-th input with the $j$-th output. The trainable parameters in the GCNN layer are the weight vectors $\mathbf{w}_{ij}$ and the biases $b_j.$ As other network architectures, GCNNs are typically structured hierarchically, with several GCNN layers followed by fully connected layers. We assume that the input of the first GCNN layer are already functions on the discretized group. 

Let $L$ be the number of GCNN layers and assume that the respective numbers of GCNN units in the layers $1,\ldots,L$ is denoted by $m_1,\ldots,m_L.$ To derive a recursive formula for the GCNN, we denote the inputs of the GCNN by $h_{0,1},\ldots,h_{0,m_0}.$ Note that they are also assumed to be functions $G^r \to \mathbb{R}.$ If for a given $\ell=0,\ldots, L-1,$ $h_{\ell,1}, \ldots, h_{\ell,m_\ell}$ are the outputs (also known as feature maps) of the $\ell$-th GCNN layer, then, the $m_{\ell+1}$ outputs of the $(\ell+1)$-st GCNN layer are given by 
\begin{align}
\label{cu_GCNN}
  h_{\ell+1, j}=\sigma \bigg(\sum_{i=1}^{m_\ell}   \mathcal{K}_{\mathbf{w}_{ij}^{(\ell)}} * h_{\ell,i} - b_j^{(\ell)}\bigg)
\end{align}
for $=1, \dots, m_{\ell+1}$.
The final output of the GCNN is given by
\begin{equation}
\label{final_pooling}
\sum_{i=1}^{m_L} \sum_{g \in G^r} h_{L,i}(g), 
\end{equation}
which equals up to reweighting a global average pooling operation. As we take the sum over all group elements $g$, this operation makes the network invariant instead of equivariant to geometric transformation (see, e.g., \citep{kondor2018generalization, keriven2019universal, bekkers2018roto}).

In our work, we consider the ReLU activation function $\sigma(x)=\max\{x,0\}.$ We denote the class of ReLU GCNNs with \( L \) layers, \( m_i \) units in layer \( i=0,\ldots, L \), \( k \) dimensional weight vectors in \eqref{eq.bcusdyb}, and $r$ the cardinality of the discretized group by
\begin{equation}
\label{GCNNclass}
\mathcal{H}(k, m_0, \ldots, m_L, r).  
\end{equation}

The learnable parameters are the vectors $\mathbf{w}_{ij}^{(\ell)}$ and the biases $b_j^{(\ell)}$ for $j=1,\ldots, m_\ell,$ $\ell=0,\ldots,L-1.$
During the training phase, they are updated through gradient-based optimization techniques, such as stochastic gradient descent (SGD). The updates aim to minimize an objective function, typically a loss function measuring the difference between the network predictions and the true labels. While in practice, GCNN architectures also include feedforward layers, we only focus in this work on the GCNN layers. 

\subsection{CNNs as a specific case of GCNNs}

The convolutional layer in CNNs can be obtained as a specific case of the GCNN layer for the translation  group \(T\). An element \( \mathbf{t} \) of the translation group corresponds to a vector \( (t_1, t_2) \in \mathbb{R}^2 \), and the group operation is defined as \( \mathbf{t} \circ \mathbf{t'} = (t_1 + t'_1, t_2 + t'_2) \), meaning that one vector is shifted by the components of another. The inverse of a translation \( \mathbf{t} \) is \( \mathbf{t}^{-1}=(-t_1, -t_2) \), reversing the direction of the shift.

A square image can be interpreted as a function on the unit square \([0,1] \times [0,1]\). Setting the function values to zero outside the unit square, it can then be extended to a function on $\mathbb{R}^2.$ Since $\mathbb{R}^2$ is isomorphic to the translation group, a square image can thus be viewed as a function on the translation group \(T\).

To illustrate the discretization step, consider the MNIST dataset~\citep{lecun1998gradient}, which consists of grayscale images of handwritten digits ranging from 0 to 9. Each MNIST image is represented by grayscale values. This means that $m_0$ is equal to 1, as the image is characterized solely by pixel brightness. In contrast, for an RGB image, $m_0=3$, corresponding to the red, green, and blue components of each pixel.

Furthermore, since the values are only on a \( 28 \times 28 \) pixel grid, the translation group \(T\) is discretized by
\[
T^{784} = \left\{ \left(\frac{i}{28}, \frac{j}{28}\right) \big| \,  i, j \in [28] \right\}.
\]
In turn, we can view an MNIST image as a function $f$ on $T^{784}.$ Note that the function value $f(\tfrac i{28}, \tfrac j{28})$ is the $(i,j)$-th pixel value.

In CNNs, the coefficients of the convolutional filters are typically represented by $s \times s$ weight matrices, with $s$ a prechosen integer. For simplicity, we consider $s=3$ in the following. The convolutional filter computes
\begin{align*}
\sum_{\ell, k=-1}^{1} w_{\ell+1, k+1} 
    f\Big(\frac{i+\ell}{28},\frac{j+k}{28}\Big), \quad i,j=1,\ldots, 28.
\end{align*}
For the kernel
 \begin{align*}
&\mathcal{K}_\mathbf{w}(u_1,u_2) = \sum_{\ell, k=-1}^{1}  w_{\ell+1, k+1} \, \mathbf{1}\Big((u_1, u_2)=\left(\frac{\ell}{28},\frac{k}{28}\right)\Big),
\end{align*} 
and $\mathbf{s}_{i,j}:=(i/28,j/28),$ this can also be rewritten in the form \eqref{convolution_discrete2} via 
\begin{align*}
    (\mathcal{K}_{\mathbf{w}} * f)(\mathbf{s}_{i,j}) &=  \sum_{\mathbf{t}\in T^{784}}\mathcal{K}_{\mathbf{w}}( \mathbf{s}_{i,j}^{-1} \circ \mathbf{t}) \cdot f(\mathbf{t}) \\
    &= \sum_{\mathbf{t}\in T^{784}}
    \sum_{\ell, k=-1}^{1}  w_{\ell+1, k+1} \cdot \mathbf{1}\Big(\left(t_1-\frac{i}{28}, t_2-\frac{j}{28}\right)\\
    &= \left(\frac{\ell}{28}, \frac{k}{28}\right)\Big) \cdot f(\mathbf{t}) \\
    &= \sum_{\ell, k=-1}^{1} w_{\ell+1, k+1} 
    f\Big(\frac{i+\ell}{28},\frac{j+k}{28}\Big).
\end{align*}
Thus, in this equivalence the kernel is a linear combination of indicator functions. 

\subsection{Comparison of GCNNs and deep feedforward neural networks}

GCNNs and deep feedforward neural networks (DNNs) differ in how their computational units are defined. In a DNN, each unit computes an affine transformations applied to the output of the previous layer, followed by an activation function \(\sigma\). If $\mathbf{z}$ is the output of the previous layer, $\mathbf{w}$ is a weight vector, and $b$ a bias, then a unit in a DNN computes 
\begin{align}
    \sigma \big(  \mathbf{w}^\top \mathbf{z}  - b \big).
    \label{eq.djsvgv}
\end{align}

The class of ReLU DNNs  with $L$ layers (that is, \(L-1\) hidden layers and one output layer), \(m_i\) units (or neurons) in layer \(i=0,\ldots, L\), and a single unit in the output layer (i.e., $m_L=1$) is denoted by
\begin{equation}
\label{DNNclass}
\mathcal{F}(m_0, \ldots, m_L).
\end{equation}
While both the DNN class $\mathcal{F}(m_0, \ldots, m_L)$ and the  GCNN class $\mathcal{H}(k, m_0, \ldots, m_L, r)$ share the same architectural structure and apply the ReLU activation function, they differ in their unit definition, meaning that $\mathcal{F}(m_0, \ldots, m_L)$ uses DNN units \eqref{eq.djsvgv} and $\mathcal{H}(k, m_0, \ldots, m_L, r),$ employs GCNN units \eqref{cu_GCNN2}.

\section{VC dimension of GCNNs}

We now derive upper bounds for the VC dimension of the GCNN class \(\mathcal{H}(k, m_0, \ldots, m_L, r)\).
We begin by formally introducing the VC dimension. 
\begin{definition}[Growth function, VC dimension, shattering]
Let $\mathcal{H}$ denote a class of functions from $\mathcal{F}$ to $\{-1, 1\}$ (often referred to as the hypotheses class). For any non-negative integer $m$, we define the growth function of $\mathcal{H}$ as the maximum number of distinct classifications of $m$ samples that can be achieved by classifiers from $\mathcal{H}$. Specifically, it is defined as:
\[
\Pi_\mathcal{H}(m) := \max_{f_1, \ldots, f_m \in \mathcal{F}} \left| \left\{ (h(f_1), \ldots, h(f_m)) : h \in \mathcal{H} \right\} \right|.
\]
If $\mathcal{H}$ can generate all possible $2^m$ classifications for a set of $m$ inputs, we say that $\mathcal{H}$ shatters that set. Formally, if
\[
\left| \left\{ (h(f_1), \ldots, h(f_m)) : h \in \mathcal{H} \right\} \right| = 2^m,
\]
we say $\mathcal{H}$ shatters the set $\{f_1, \ldots, f_m\}$.
\\
\\
The Vapnik-Chervonenkis dimension of $\mathcal{H}$, denoted as $\VC(\mathcal{H})$, is the size of the largest shattered set, specifically the largest $m$ such that $\Pi_\mathcal{H}(m) = 2^m$. If no such largest $m$ exists, we define $\VC(\mathcal{H}) = \infty$.
\end{definition} 

The VC dimension cannot be directly applied to a class of real-valued functions, such as neural networks. To address this, we follow the approach in \citep{bartlett2019nearly} and use the pseudodimension as a measure of complexity. The pseudodimension is a natural extension of the VC dimension and retains similar uniform convergence properties (see \citep{pollard1990empirical} and Theorem 19.2 in \citep{anthony1999neural}).

\begin{definition}[pseudodimension]
For a class \( \mathcal{H} \) of real-valued functions, we define its pseudodimension as \( \VC(\mathcal{H}) := \VC(\mathrm{sign}(\mathcal{H})) \), where
\[
\mathrm{sign}(\mathcal{H}) := \{ \mathrm{sign}(H - b)\mid H \in \mathcal{H},\ b\in\mathbb{R} \},
\]
where $\mathrm{sign}(x)$ is $1$ for $x>0$ and $-1$ otherwise. We write $\Pi_{\mathcal{H}}$ to denote a growth function of $\mathrm{sign}(\mathcal{H})$.
\end{definition}

For common parametrized function classes, the VC dimension relates to the number of parameters. For DNNs, the VC dimension further depends on the network depth, as discussed in \citep{bartlett2019nearly}.

The following result provides an upper bound on the VC dimension of the GCNN class \(\mathcal{H}(k, m_0, \ldots, m_L, r)\). Recall that \( L \) is the number of layers, \( m_i \) is the number of units in layer \( i=0,\ldots, L \),  weight vectors are \( k \) dimensional, and $r$ is the cardinality of the discretized group.

\begin{theorem}[Upper Bound]
\label{upper_bound}
The VC dimension of the GCNN class \(\mathcal{H}=\mathcal{H}(k, m_0, \ldots, m_L, r)\) is upper bounded by
\begin{equation}
\label{upper_bound_equation_gcnn}
UB(\mathcal{H}):=L+1 + 4 \bigg(\sum_{\ell=1}^L W_\ell\bigg) \log_2\bigg(8e r \sum_{\ell=1}^L m_\ell \bigg),
\end{equation}
with \(W_\ell\) the number of parameters up to the \(\ell\)-th layer, that is,
\begin{equation}
\label{number weights GCNN}
W_\ell := \sum_{j=1}^\ell m_j (k m_{j-1} + 1).
\end{equation}
\end{theorem}
To verify \eqref{number weights GCNN}, observe that each unit in layer \(i\), has \(k m_{i-1}\) weight parameters and one bias. A proof sketch of this theorem is provided in Section \ref{proofs}, with a detailed proof in the supplement. 
\\
The VC bound generalizes those known for vanilla CNNs, as shown in Lemma 12 in the supplement of \citep{kohlerwalter2023}, where similar dependencies between VC dimension and number of parameters and hidden layers were derived. We further compare the sample complexity of GCNNs with DNNs. For that we rely on the VC dimension bound for DNNs with piecewise-polynomial activation functions that has been derived in Theorem 7 of \citep{bartlett2019nearly}. Specifically, we focus on the class of DNNs with \(L\) layers and \(m_i\) units in layer \(i\), as defined in \eqref{DNNclass}. This network class
corresponds to the case where $d=1$ and $p=1$ in Theorem 6 of \citep{bartlett2019nearly}. By applying the inequality $\log_2(\log_2 (x))\leq \log_2(x)$ that holds for any $x\geq 2$, the VC dimension bound for this class becomes
\begin{align*}
UB(\mathcal{F}):=L + 2\bigg(\sum_{\ell=1}^{L} W_\ell(\mathcal{F})\bigg)\log_2\bigg(4e\sum_{\ell=1}^L \ell m_\ell\bigg), 
\end{align*}
where $\mathcal{F}$ is used as shorthand notation for $\mathcal{F}(m_0,\ldots, m_L)$. Here $W_i(\mathcal{F})$ represents the number of parameters of the class $\mathcal{F}$ up to the $i$-th layer, that is,
\begin{equation}
\label{number_weights_DNN}
W_\ell(\mathcal{F})=\sum_{j=1}^\ell m_{j}(m_{j-1} + 1). 
\end{equation}
By comparing \eqref{number weights GCNN} and \eqref{number_weights_DNN}, and assuming an equal number of computational units per layer for both, GCNNs and DNNs, we observe that the number of parameters in GCNNs satisfies the inequality $W_i \leq k W_i(\mathcal{F})$, with $k$ the dimension of the weight vector $\mathbf{w}$ in \eqref{eq.bcusdyb}.
\\
\\
This together with the bound in Theorem \ref{upper_bound} then yields
\begin{align}
\label{comparison_UB}
UB(\mathcal{H}) &\leq 2k UB(\mathcal{F}) + 4 \bigg(\sum_{\ell=1}^{L} W_\ell(\mathcal{H})\bigg) \log_2(2\resolution).
\end{align}
An alternative to bounding the VC dimension based on the number of layers and neurons per layer is to express the bound in terms of the total number of trainable parameters. The main advantage of this approach is that it applies to a wider range of architectures including sparsely connected GCNNs.
In this context, \cite{bartlett2019nearly} establishes a bound on the VC dimension for the class 
\begin{align}
\label{definition_DNN_L_W}
\mathcal{F}_{W,L}\coloneqq\{\mathcal{F}=\mathcal{F}(m_0, \ldots, m_\ell)\mid \ell\leq L, W_L(\mathcal{F})\leq W\},    
\end{align}
consisting of DNNs with at most $L$ hidden layers and an overall number $W$ of weights. In particular, they show that there exist universal constants $c$ and $C$ such that
\begin{align}
\label{VCDNN}
c \cdot W L \log\bigg(\frac{W}{L}\bigg) &\leq \VC(\mathcal{F}_{W,L})  \nonumber \\ 
&\leq \max_{\mathcal{F}\in\mathcal{F}_{W,L}}UB(\mathcal{F}) \\ &\leq C \cdot W L \log W. \nonumber
\end{align}
Similarly, we consider the GCNN class
\begin{align}
\label{definition_GCNN_L_W}
&\mathcal{H}_{W,L, r}\coloneqq \{\mathcal{H}(k, m_0, \ldots, m_\ell, r) \mid \ell\leq L, W_L\leq W\},
\end{align}
consisting of all GCNN architectures with a total number of parameters bounded by \(W\), depth at most \(L\), 
and \(r\) the cardinality of the discretized group.

Inequality~\eqref{comparison_UB}, along with the lower and upper bounds on the VC dimension of DNNs in \eqref{VCDNN} leads to the following result.
\begin{corollary}
\label{UB_VC_H_W_l_and_F_W_l}
Let \(W\) be the total number of parameters and \(L\) the depth of the network, with \(L < W^{0.99}\). There exists a universal constant $C$ such that
\begin{align*}
\VC(\mathcal{H}_{W,L,r}) \leq C \left(\VC(\mathcal{F}_{W,L}) + LW \log_2(r)\right).
\end{align*}
\end{corollary}
This bound obtains the nearly optimal rate as shown in the next theorem.
\begin{theorem}[Lower bound]
\label{lower_bound}
If $W, L>3,$ then there exists a universal constant \(c\) such that
\begin{align*}
\VC(\mathcal{H}_{W,L,r}) &\geq c \max\{ \VC(\mathcal{F}_{W,L}), W \log_2(\resolution)\}.
\end{align*}
\end{theorem}

By combining Corollary~\ref{UB_VC_H_W_l_and_F_W_l} and Theorem~\ref{lower_bound}, we conclude that there exist universal constants \(c\) and \(C\) such that
\begin{align*}
c\Big(\VC(\mathcal{F}_{W,L}) &+ W\log_2(\resolution)\Big) \\
&\leq  \VC(\mathcal{H}_{W,L,r})\\
&\leq C \Big( \VC(\mathcal{F}_{W,L}) + LW\log_2(r)\Big).
\end{align*}
In practice, the depth in GCNNs is rather small, e.g., $L=7$ in the initial article of \cite{cohen2016group} for the rotated MNIST dataset. In this regime, the obtained rates in the upper and lower bound of the VC dimension nearly match.

\section{Proof sketch of Theorem~\ref{upper_bound} }
\label{proofs}
To derive an upper bound for the VC dimension of the GCNN class $\mathcal{H}=\mathcal{H}(k, m_0, \ldots, m_L, r),$ recall that the parameter $r$ is the cardinality of the discretized group $G^r \coloneqq \{g_1, g_2, \ldots, g_r\}$.  The parameters $k,m_0,\ldots,m_L$ define the network architecture. The class $\mathcal{H}$ consists therefore of all functions $\mathbf{w} \mapsto h_\mathbf{w}$ that can be represented by a GCNN with this architecture but different network parameter $\mathbf{w} \in \mathbb{R}^{W}.$ 

The proof of Theorem~\ref{upper_bound} closely follows the proof of Theorem 7 of~\citep{bartlett2019nearly}, with adjustments to account for the structure of GCNN-units. To find an upper bound of a fixed number \(m\) on the VC dimension, we need to show that neural networks in \(\mathcal{H}\) cannot shatter any set of \(m\) functions. 

Specifically, consider an input set of $m$ functions from $G^r$ to $\mathbb{R}^{m_0}$, denoted as
\begin{align}
\label{Fm}
    F_m \coloneqq \{f_1, \ldots, f_m\}.
\end{align}
We aim to bound the number of distinct sign patterns that networks in $\mathcal{H}$ can generate when applied to functions in $F_m$ i.e., 
\begin{align*}
    |\{\text{sign}(h_{\mathbf{w}}(f_1)), \dots, \text{sign}(h_{\mathbf{w}}(f_m)): \mathbf{w} \in \mathbb{R}^{W}\}|.
\end{align*}
By bounding the number of distinct sign patterns, we obtain an upper bound on the growth function $\Pi_{\mathcal{H}}$ of $\mathrm{sign}(\mathcal{H})$. If $\Pi_{\mathcal{H}} < 2^m$, then $m$ is an upper bound for the VC dimension of $\mathcal{H}$.

For fixed network architecture and fixed input, the output of the GCNN-units depends only on the network parameters. Therefore, to understand the possible sign patterns that the network can generate on \(F_m\), we need to analyze the dependence of the GCNN-units on the network parameters. 

Since the ReLU activation is piecewise linear, and the sum of piecewise linear functions remains piecewise linear, each GCNN unit in any layer $\ell$, as defined in \eqref{cu_GCNN}, can be viewed as a composition of $\ell$ piecewise linear functions. This results in a piecewise polynomial function of degree $\leq \ell$ with respect to the network parameters up to layer \(\ell\). Consequently, for each layer $\ell$,  the parameter space $\mathbb{R}^{W_{\ell}}$ can be partitioned into regions $\{P_1, \dots, P_{S(\ell)}\}$, where within each region, the GCNN units in the $(\ell+1)$-st layer behave like a fixed polynomial function in $W_{\ell}$ variables of $\mathbf{w}$, of degree at most $\ell$.

As a first step to prove Theorem~\ref{upper_bound}, we recursively find a bound for $S(\ell)$ and then 
determine how many different sign patterns the classifiers in \(\mathrm{sign}(\mathcal{H})\) can generate within each of these regions. The following lemma establishes how the upper bound for \(S(\ell)\) evolves from \(S(\ell-1)\). The proof is provided in the supplementary material. 
\begin{lemma}
\label{number_of_partitions}
Let $\mathcal{H}$ be the class of GCNNs defined in~\eqref{definition_GCNN_L_W} with $\leq W_{\ell}$ network parameters  up to layer $\ell \in \{1, \dots, L\}.$ If $F_m$ is the class of functions defined in \eqref{Fm} and $S(\ell)$ is as defined above, then, for $\ell=0,1,\ldots,L-1,$
\begin{equation}
    S(\ell+1) \leq 2\bigg(\frac{2em_{\ell+1} m r(\ell+1)}{W_{\ell+1}}\bigg)^{W_{\ell+1}} S(\ell).
\end{equation}
Moreover, the GCNN units \(\{h_{\ell+1, j}(g) \mid j \leq m_{\ell+1}, f \in F_m, g \in G^r\}\) with $h_{\ell+1, j}$ defined for different functions $f \in F_m$ in \eqref{cu_GCNN3} and \eqref{cu_GCNN} are piecewise polynomials of degree \( \leq \ell+1 \) in the network parameters. 
\end{lemma}

Next, by applying Lemma 1 in \citep{bartlett1998almost}, which provides an upper bound on the number of distinct sign patterns that can be generated by polynomials of finite degree, we obtain a bound on the growth function $\Pi_{\mathcal{H}}$.

The proof is provided in the supplementary material.
\begin{lemma}
\label{growth_function}
Let $\mathcal{H}$ be the class of GCNNs defined in \eqref{GCNNclass} with at most $W_{\ell}$ of parameters up to layer $\ell \leq L$ and $m_{\ell}$ GCNN units in layer $\ell$. Let $m>0$ be an integer, then
\begin{equation*}
\Pi_{\mathcal{H}}(m)\leq 2^{L}\prod_{\ell=1}^{L}\left(\frac{2e m r m_\ell \ell}{W_\ell}\right)^{W_\ell}  2\left(\frac{2e m L}{W_{L}+1}\right)^{(W_L+1)}. 
\end{equation*}
\end{lemma}
The relationship between VC-dimension and growth functions yields Theorem~\ref{upper_bound}. A complete proof can be found in supplementary material.

\section{Proof Sketch of Theorem~\ref{lower_bound}}
To establish a lower bound for the VC dimension of the GCNN class, we recall that \(\mathcal{H}_{W,L,r}\) as defined in \eqref{definition_GCNN_L_W} represents the class of GCNNs with resolution \(r\), $k$ the dimension of the kernel space, at most \(L\) layers, and no more than \(W\) parameters. Additionally, let \(G^r \coloneqq \{g_1, g_2, \dots, g_r\}\) be a discretized group containing the identity element \(e\). 

To derive a lower bound, we aim to prove that the VC dimension of \(\mathcal{H}_{W,L,r}\) exceeds a given integer \(m\). For this, it is sufficient to find networks in \(\mathcal{H}_{W,L,r}\) that can shatter a set of \(m\) input functions.

In particular, to prove Theorem~\ref{lower_bound}, one needs to show the existence of a universal constant \(c > 0\) such that
\begin{equation}
\label{eqdeep}
\VC(\mathcal{H}_{W,L,r}) \geq c \cdot \VC(\mathcal{F}_{W,L}),    
\end{equation}
and
\begin{equation}
\label{second_part_lower_bound}
\VC(\mathcal{H}_{W,L,r}) \geq c \cdot W \lfloor \log_2{r} \rfloor.
\end{equation}
Here, \(\mathcal{F}_{W,L}\) is the class of DNNs with at most \(L\) hidden layers and at most \(W\) weights, as defined in \eqref{definition_DNN_L_W}.

To prove \eqref{eqdeep}, the first step is to establish a connection between DNNs and GCNNs. The next lemma demonstrates how a DNN can be associated with a GCNN with the same number of parameters and layers. Specifically, it states that when a DNN is evaluated on inputs \(f(g)\), where \(f \in F_m\) and \(g \in G^r\) (with \(F_m\) defined as in \eqref{Fm}), the sum of its outputs, taken over different elements of $G^r$, matches the output of the corresponding GCNN for the same function \(f\). 

Before stating the lemma, recall that $\mathcal{F}(m_0, \ldots, m_L)$, as defined in \eqref{DNNclass}, denotes the class of fully connected feedforward ReLU networks with $L$ layers, where $m_i$ denotes the number of units in the $i$-th layer for $i=1,\ldots,L$.
The outputs of the last hidden layer of any neural network \(\tilde{h}_{\mathbf{w}} \in \mathcal{F}(m_0, \ldots, m_L)\) with parameters $\mathbf{w}$ can be represented as a vector of size $m_L$, this is,   \((\tilde{h}_{\mathbf{w}}^{(1)}, \ldots, \tilde{h}_{\mathbf{w}}^{(m_L)})\).

\begin{lemma}
\label{connection_DNN_GCNN}
Consider GCNNs where the G-correlation in \eqref{convolution_discrete2} uses kernels from a one-dimensional vector space with a fixed basis given by the indicator function of the identity element \(e\). For every \(\tilde{h}_{\mathbf{w}} \in \mathcal{F}(m_0, \dots, m_L)\) there exists a GCNN \(h_{\mathbf{w}}\), with the same number of channels in each layer, i.e., \(h_{\mathbf{w}} \in \mathcal{H}(1, m_0, \dots, m_L, r)\) and parameters $\mathbf{w}$, such that for any input function \(f: G^r \to \mathbb{R}^{m_0}\)
\[
\sum_{i=1}^{m_L} \sum_{j=1}^{r} \tilde{h}_{\mathbf{w}}^{(i)}(f(g_j)) = h_{\mathbf{w}}(f).
\]
\end{lemma}
The proof of this lemma is provided in the supplementary material. The lemma implies that instead of finding GCNNs that shatter \( F_m \), we can also find DNNs in the class $\mathcal{F}(m_0, \dots, m_L)$ shattering $F_m$. For the construction of this DNN architecture we define sums of so-called 'indicator' neural networks, i.e., DNNs with ReLU activation functions that approximate indicator functions over a specified interval. The endpoints of this interval act as parameters of the neural networks, allowing to adjust the interval by modifying these parameters.

For \(a < b\) and \(\epsilon > 0\), the shallow ReLU network
\begin{align}
\label{indicator_function}
    \mathbf{1}_{(a,b,\epsilon)}(x) 
    &=\frac{1}{\epsilon}\Big(\big(x - (a-\epsilon)\big)_{+} - \big(x - a\big)_{+} \nonumber\\
    &\quad\quad + \big(x - b\big)_{+} - \big(x - (b+\epsilon)\big)_{+}\Big),
\end{align}
with four neurons in the hidden layer and \((x)_+ := \max(x, 0)\) the ReLU activation function, approximates the indicator function on $[a,b]$ in the sense that $\mathbf{1}_{(a,b,\epsilon)}(x)$ is $1$ if  $x \in [a, b]$ and $0$ if $x < a - \epsilon \text{ or } x > b + \epsilon$.

To show \eqref{eqdeep}, we construct a DNN architecture from $m_0+1$ smaller DNN classes. One of these DNN classes is capable of shattering a set of \( \VC(\mathcal{F}_{W,L}) \) vectors in \( \mathbb{R}^{m_0} \), while the remaining \( m_0 \) DNN classes produce 'indicator' networks, ensuring that the combined DNNs vanish outside a certain \( m_0 \)-dimensional hypercube. We show, that this class of networks shatters a set of \( \VC(\mathcal{F}_{W,L}) \) input functions from the set \( \{f: G^r \rightarrow \mathbb{R}^{m_0}\} \). By applying Lemma \ref{connection_DNN_GCNN}, we further show that a class of GCNNs with at most \( 5W \) weights and \( L \) layers can shatter the same set of \( \VC(\mathcal{F}_{W,L}) \) input functions. Finally, we use that \( \VC(\mathcal{F}_{tW, L}) \geq c_t \cdot \VC(\mathcal{F}_{W,L}) \) for a constant \( t \) with \( c_t < 1 \).

\begin{lemma}
\label{construction_Vn_DNN}
For \( L > 3 \) let $\mathcal{H}_{5W, L+1, r}$ be the class of GCNNs defined as in \eqref{definition_GCNN_L_W} and $\mathcal{F}_{W, L}$ be the class of DNNs defined as in \eqref{definition_DNN_L_W}. Then
\[
\VC(\mathcal{H}_{5W, L+1, r}) \geq \VC(\mathcal{F}_{W, L}).
\]
\end{lemma}

\begin{corollary}
\label{LB_FNN}
In the setting of Lemma \ref{construction_Vn_DNN}, if the number of layers \( L > 3 \) and \( L \leq  W^{0.99} \), then there exists a constant \( c \) such that 
\[
\VC(\mathcal{H}_{W, L,r}) \geq c \cdot \VC(\mathcal{F}_{W, L}).
\]
\end{corollary}

To prove the lower bound in \eqref{second_part_lower_bound}, we first show that an 'indicator' neural network class can shatter input functions \( F_m \subset \{ f: G^r \rightarrow \mathbb{R} \}\) with \( m \coloneqq \lfloor \log_2{r} \rfloor\), by adjusting its parameters, i.e., the interval endpoints to the values of the input functions from \(F_m\). By using Lemma \ref{connection_DNN_GCNN} again, we then find a corresponding GCNN, which also shatters \(F_m\). This shows, that for any two numbers \( A < B \), the GCNN can serve as an indicator function, outputting zero for any input function that maps outside the interval \([A, B]\). This ensures that the GCNN only shatters functions within the specified range. The next lemma provides a formal statement. The proof is provided in the supplementary material.

\begin{lemma}
\label{log(d)_construction} 
Let \(\mathcal{H}_{4, L, r} \) be the class of GCNNs defined as in \eqref{definition_GCNN_L_W}. Then
\begin{align*}
    VC(\mathcal{H}_{4, L, r}) \geq \lfloor \log_2{r} \rfloor.
\end{align*} 
Moreover, for any two numbers \( A < B \), there exists a finite subclass of GCNNs \( \mathcal{H} \subset \mathcal{H}_{4, L, r} \) 
that shatters a set of \(\log_2 r\) input functions 
\[
\{ f_i : G^r \rightarrow [A, B] \mid i = 1, \ldots, \log_2 r \},
\]
and
outputs zero for any input function \( f : G^r \rightarrow \mathbb{R} \setminus [A, B] \).
\end{lemma}

By using that each class $\mathcal{H}_{4, L,r}$ can shatter a set of functions
\begin{align*}
    \{f_i: G^r \to [A_j,B_j]|i=1, \dots, \lfloor\log_2 r\rfloor\}
\end{align*}
with disjoint intervals $[A_j, B_j]$, $j=1, \dots, \tilde{W}$, we can find a class of GCNNs that shatters a set of functions with \( \tilde{W} \lfloor \log_2{r} \rfloor \) elements. By choosing $W=4\tilde{W}$, this shows that 
\[
\VC(\mathcal{H}_{W,L,r}) \geq \frac{1}{4}W \lfloor \log_2{r} \rfloor.
\]

The following corollary is a formal statement of the previous inequality. It is proved in the supplementary material.

\begin{corollary}
\label{LB_resolution}
The VC dimension of the class \( \mathcal{H}_{4W,L, r} \), consisting of GCNNs with \( 4W \) weights, \( L \) layers, and resolution $r$ satisfies the inequality 
\[
\VC\big(\mathcal{H}_{4W,L, r}\big) \geq W \lfloor \log_2{r} \rfloor.
\]
\end{corollary}

By combining Corollaries~\ref{LB_FNN} and ~\ref{LB_resolution}, we obtain the lower bound
\[
\VC(\mathcal{H}_{W,L,r}) \geq  \max\Big\{c \VC(\mathcal{F}_{W,L}), \frac{1}{4} W \lfloor \log_2{r} \rfloor\Big\},
\]
thereby proving Theorem~\ref{lower_bound}.

\section{Conclusion}
In this work, we established nearly-tight VC dimension bounds for a class of GCNNs. The bounds reveal that the complexity of GCNNs depends on the number of layers, the number of weights, and the resolution of the group acting on the input data. 

While the VC dimensions of GCNNs and deep neural networks (DNNs) are similar, for GCNNs an extra term \(W\lfloor \log_2 r \rfloor\) occurs that increases logarithmically in the resolution $r$ of the input data.  This finding aligns with previous work~\citep{petersen2024vc} showing that, when the resolution of the group approaches infinity, the VC dimension of the GCNN becomes infinite as well. The logarithmic scaling with respect to the resolution highlights the sensitivity of GCNNs to the discretization of the group, providing deeper insight into how data impacts the model complexity and generalization capabilities.

\subsubsection*{Acknowledgements}
J. S.-H. was supported in part by the NWO Vidi grant VI.Vidi.192.021. S.L. was supported in part by the NWO Veni grant VI.Veni.232.033

\appendix

\section{Supplementary material}

\subsection{Proof of Theorem 3.3}
Recall that 
\begin{equation}
\label{classGCNN}
\mathcal{H} = \mathcal{H}(k, m_0, \ldots, m_L, r),
\end{equation}
where \( r \) is the cardinality of the discretized group \( G^r \coloneqq \{g_1, g_2, \dots, g_r\} \). The parameter \( k \) determines the number of basis functions
\begin{equation}
\texttt{K}_s: G^r \rightarrow \mathbb{R}, \quad s = 1, \dots, k,
\end{equation}
in the parametrization of the kernel function
\[
\mathcal{K}_{\mathbf{w}} = \sum_{s=1}^{k} w_s \texttt{K}_s.
\]
The other parameters \( m_0, \dots, m_L \) define the network architecture, and \( W_\ell \) represents the number of parameters in the GCNN up to layer \( \ell \). The class \( \mathcal{H} \) consists of all functions that can be represented by a neural network with this architecture.

We restate Theorem 3.3 for convenience:

\begin{theorem}[Theorem 3.3]
\label{Theorem 3.3}
The VC dimension of the GCNN class \( \mathcal{H} = \mathcal{H}(k, m_0, \dots, m_L, r) \) with $r>1$, is bounded from above by
\begin{equation}
\label{ubh}
UB(\mathcal{H}) := L + 1 + 4 \left( \sum_{\ell=1}^L W_\ell \right) \log_2 \left( 8e r \sum_{\ell=1}^L m_\ell \right).
\end{equation}
\end{theorem}

For the proof, we consider an input consisting of \( m \) functions from \( G^r \) to \( \mathbb{R}^{m_0} \), denoted by
\begin{equation}
\label{Fmsupp}
F_m \coloneqq \{f_1, \ldots, f_m\}.
\end{equation}

To prove Lemma 4.1, we use the following known result:
\begin{lemma}
\label{possible_sign}
[Lemma 1,~\cite{bartlett1998almost}]
Let $p_1, \ldots, p_{\tilde{m}}$ be polynomials of degree at most $t$ depending on $n \leq \tilde{m}$ variables. Then
\begin{align*}
  \Pi &:= \left| \left\{ \left( \mathrm{sign}(p_1(x)), \ldots, \mathrm{sign}(p_m(x)) \right) : x \in \mathbb{R}^n \right\} \right| \leq 2\left(\frac{2e\tilde{m} t}n\right)^n.
\end{align*}
\end{lemma}
Recall that \( S(\ell) \) is the number of regions in the parameter space $\mathbb{R}^{W_{\ell}}$, such that in each region, the GCNN units in the \(\ell\)-th layer (denoted by \( \{h_{\ell, j}(g) \mid j \leq m_{\ell}, f \in F_m, g \in G^r\} \)) behave like a fixed polynomial of degree at most \(\ell\) in the \( W_\ell \) network parameters that occur up to layer $\ell$. 

\begin{lemma}
\label{Lemma 4.1}
[Lemma 4.1]
Let \( \mathcal{H} \) be the class of GCNNs defined in~\eqref{classGCNN}, with at most \( W_\ell \) parameters up to layer \( \ell \in \{1, \dots, L\} \). If \( F_m \) is the class of functions defined in \eqref{Fmsupp}, and \( S(\ell) \) is as defined above, then for \( \ell = 0, 1, \dots, L-1 \),
\begin{equation}
\label{Sl}
S(\ell+1) \leq 2\left(\frac{2e m_{\ell+1} m r (\ell+1)}{W_{\ell+1}}\right)^{W_{\ell+1}} S(\ell).
\end{equation}
Moreover, the GCNN units \(\{h_{\ell+1, j}(g) \mid j \leq m_{\ell+1}, f \in F_m, g \in G^r\}\) with $h_{\ell+1, j}$ defined for different functions $f \in F_m$, are piecewise polynomials of degree \( \leq \ell+1 \) in the network parameters. 
\end{lemma}

\begin{proof}
As a first step of the proof, we show that any GCNN unit $h_{\ell,j}$ of any layer $\ell \in \{0, \dots, L\}$ and $j \in \{1, \dots, m_{\ell}\}$ forms a piecewise polynomial of degree at most $\ell$. We proceed by induction on the layers \( \ell \). 

For the base case \( \ell = 0 \), the GCNN units \( h_{0,j} \) for \( j \leq m_0 \) correspond to the input of the network, which is independent of the network parameters. Therefore, $h_{0,j}$ are polynomials of degree $0$.

Assume the statement holds for all layers up to $\ell$. We now prove it for layer $\ell + 1$. The GCNN unit in layer \( \ell+1 \) is defined by a convolution with the feature maps from the previous layer, that is,
\[
h_{\ell+1, j} = \sigma\left(\sum_{i=1}^{m_\ell} \mathcal{K}_{\mathbf{w}_{ij}^{(\ell)}} * h_{\ell,i} - b_j^{(\ell)}\right),
\]
where the convolutional filter is expanded in terms of the fixed basis functions \( \texttt{K}_s \) via \( \mathcal{K}_{\mathbf{w}} = \sum_{s=1}^k w_s \texttt{K}_s \). For fixed network parameters, \( g \mapsto h_{\ell,j}(g) \) is a function of the group, with \( h_{\ell,j}(g) \in \mathbb{R} \) for any \( g \in G^r \). By the induction hypothesis, $h_{\ell,j}(g)$ are piecewise polynomials of degree at most $\ell$ with respect to the network parameters, with the polynomial pieces depending on the network input and the group element $g$. 

Next, observe that for any input and any group element \( g \), the term \( (\texttt{K}_s * h_{\ell, j})(g) \) can be written as 
\[
(\texttt{K}_s * h_{\ell,j})(g) = \sum_{g' \in G^r} \texttt{K}_s\big(g^{-1} \circ g'\big) \cdot h_{\ell,j}(g').
\]
Since \( h_{\ell,j}(g') \) is a piecewise polynomial of degree at most \(\ell\), it follows that \( (\texttt{K}s * h_{\ell,j})(g) \) is also a piecewise polynomial of degree at most \(\ell\). Thus, the convolution
\[
(\mathcal{K}_{\mathbf{w}} * h_{\ell,j})(g) = \sum_{s=1}^k w_s (\texttt{K}_s * h_{\ell, j})(g)
\]
is a weighted sum of piecewise polynomials, which remains a piecewise polynomial. However, multiplying by the weights $w_s$ increases the degree of the polynomial, making it at most \(\ell+1\).
Subtracting the bias term \( b_j^{(\ell)} \) and applying the ReLU activation function may increase the number of pieces, but does not increase the degree of the polynomials. Therefore, for any input and any group element \( g \), the GCNN unit \( h_{\ell+1,j}(g) \) remains a piecewise polynomial with degree \(\leq \ell+1\). This completes the proof by induction.

Next, we show \eqref{Sl}. Each GCNN unit in layer $\ell+1$ is computed by
\begin{align*}
\sigma\left(\sum_{i=1}^{m_\ell} \mathcal{K}_{\mathbf{w}_{ij}^{(\ell)}} * h_{\ell,i}\right),
\end{align*}
with $\sigma$ being the ReLU activation function $\sigma(x)=\max\{x,0\}$. As mentioned above, applying the ReLU function can increases the number of regions in the parameter space where the GCNN units behave as polynomials. This occurs, as the ReLU function either outputs the input itself (for positive values) or zero (otherwise). As a result its application decomposes each of the $S(\ell)$ regions of the parameter space in layer $\ell$ in multiple subregions. To bound this number of subregions, we need to count the number of possible sign pattern that can arise after applying the ReLU activation.

Fixing one of the $S(\ell)$ regions of layer $\ell$, by definition, all functions \( h_{\ell,j}(g) \) are polynomials in the parameters of degree at most $\ell$. Each of the $m_{\ell+1}$ GCNN unit in layer $\ell+1$, are then also a polynomial of degree at most $\ell+1$, leading to an overall amount of polynomials of $m_{\ell+1}mr$, where $m$ is the number of input function defined in the \eqref{Fmsupp} and $r$ is the resolution. Applying Lemma~\ref{possible_sign} to the $\tilde{m}=m_{\ell+1}mr$ polynomials of degree $t=\ell+1$ depending on $n=W_{\ell+1}$ parameters leads to an overall number of different sign patterns of each region of $S(\ell)$ of
\[
2\bigg(\frac{2em_{\ell+1} m r(\ell+1)}{W_{\ell+1}}\bigg)^{W_{\ell+1}}.
\]
This shows \eqref{Sl} and concludes the proof.

\end{proof} 

\begin{lemma}
\label{Lemma 4.2}
[Lemma 4.2]  
Let \( \mathcal{H} \) be the class of GCNNs defined in \eqref{classGCNN}, with at most \( W_\ell \) parameters up to layer \( \ell \leq L \), and \( m_\ell \) GCNN units in layer \( \ell \). For any integer \( m > 0 \), the growth function of this class can be bounded by
\begin{equation*}
\Pi_{\mathcal{H}}(m) \leq 2^L \prod_{\ell=1}^L \left( \frac{2e m r m_\ell \ell}{W_\ell} \right)^{W_\ell}  2 \left( \frac{2e m L}{W_L + 1} \right)^{W_L + 1}.
\end{equation*}
\end{lemma}

\begin{proof}
Lemma~\ref{Lemma 4.1} shows that after \( L \) layers, there are at most
\[
2^L \prod_{\ell=1}^{L} \left( \frac{2e m r m_\ell \ell}{W_\ell} \right)^{W_\ell}
\]
regions in the parameter space \(\mathbb{R}^{W}\), on which the GCNN units in the last layer \(\{h_{L,i}(g) \mid i \leq m_L, f \in F_m, g \in G^r\}\) behave like a fixed polynomial function of degree $\leq L$ in \( W \) variables.

Recall that the final output of the neural network, is obtained by applying average pooling to the outputs of the GCNN units in the last layer.
This implies that, for a fixed network architecture and input, the output of  the neural network is a piecewise polynomial of degree at most \( L \), depending on all \( W_L\) network parameters. Since there are \( m \) possible inputs \( f_1, \ldots, f_m \), we get \( m \) piecewise polynomials, each corresponding to one of these inputs. Bounding the growth function $\Pi_{\mathcal{H}}(m)$ now means we need to count the number of different sign patterns that arises for classifiers in \(\mathrm{sign}(\mathcal{H})\). For that, we recall that by Definition 3.2 in the main article, 
\begin{align*}
    \mathrm{sign}(\mathcal{H}):=\{\mathrm{sign}(h_{\mathbf{w}}-b)|h_{\mathbf{w}} \in \mathcal{H},\mathbf{w}\in R^{W_L},  b\in \mathbb{R}\}.
\end{align*}
Applying Lemma~\ref{possible_sign} to $m$ polynomials of the form $h_{\mathbf{w}}-b$ of degree at most $L$ and $W_L+1$ variables leads to no more than
\begin{align}
    2 \left( \frac{2e m L}{W_L + 1} \right)^{W_L + 1}
\label{eq.bcjsdh}
\end{align}
distinct sign patterns that the classifiers in \(\mathrm{sign}(\mathcal{H})\) can produce.

Thus, the growth function within each region, where the GCNN units in the last layer \(\{h_{L,i}(g) \mid i \leq m_L, f \in F_m, g \in G^r\}\) behave like a fixed polynomial function in \( W \) variables, is bounded by \eqref{eq.bcjsdh}. As a result, we conclude that the overall growth function \( \Pi_{\mathcal{H}}(m) \) is bounded by
\[
S(L) \cdot 2 \left( \frac{2e m L}{W_L + 1} \right)^{W_L + 1} = 
2^L \prod_{\ell=1}^{L} \left( \frac{2e m r m_\ell \ell}{W_\ell} \right)^{W_\ell} \cdot 2 \left( \frac{2e m L}{W_L + 1} \right)^{W_L + 1}.
\]

This completes the proof.
\end{proof}

For the proof of the Theorem~\ref{Theorem 3.3} we also use the following technical lemma

\begin{lemma}
\label{inequality}
[Lemma 16,~\cite{bartlett1998almost}]
Suppose \(2^{\tilde{m}} \leq 2^\kappa \left(\frac{\tilde{m} \cdot \tilde{r}}{\tilde{w}}\right)^{\tilde{w}}\) for some \(\tilde{r} \geq 16\) and \(\tilde{m} \geq \tilde{w} \geq \kappa \geq 0\). Then \(\tilde{m} \leq \kappa + \tilde{w} \log_2(2\tilde{r} \log_2 \tilde{r})\).
\end{lemma}

\begin{proof}[Proof of Theorem~\ref{Theorem 3.3}]
 Let \( m := \VC(\mathcal{H}) \). For convenience, define the sum 
\begin{equation}
\label{def_tilde_W}
\tilde{W} \coloneqq \sum_{i=1}^L W_i,  
\end{equation}
where $W_i$ denotes the number of parameters of a GCNN in $\mathcal{H}$ up to layer $i$.

We consider two complementary cases and prove the theorem for each of them separately.

\paragraph{Case 1:  $m < \tilde{W} + W_L + 1$.}  
In this case, we have \( \tilde{W} + W_L + 1 < 3\tilde{W} < UB(\mathcal{H}) \), where \( UB(\mathcal{H}) \) is defined in \eqref{ubh}. For the latter inequality we use that $\log_2 \left( 8e r \sum_{\ell=1}^L m_\ell \right) >1$. Therefore, Theorem~\ref{Theorem 3.3} holds.

\paragraph{Case 2: $m \geq \tilde{W} + W_L + 1$.}
Since $m$ represents the $\VC$ dimension of $\mathcal{H}$, it follows from the definition of the VC dimension (see Definition 3.1 in the main article) that \(\Pi_{\mathcal{H}}(m)=2^m \). Applying Lemma~\ref{Lemma 4.2} gives us
\begin{equation}
\label{first_ineq}
\Pi(m)=2^m \leq 
2^{L+1}\prod_{\ell=1}^{L}\left(\frac{2e m r m_\ell \ell}{W_\ell}\right)^{W_\ell} \left(\frac{2e m L}{W_L + 1}\right)^{W_L + 1}.    
\end{equation}

Next, we apply the weighted arithmetic-geometric mean (AM-GM) inequality to the right side of ~\eqref{first_ineq}, using weights \(
W_\ell/(\tilde{W} + W_L + 1)\) for \(\ell = 1,2,\ldots,L\), and \(W_L/(\tilde{W} + W_L + 1)\), where \(\tilde{W}\) is defined in \eqref{def_tilde_W}. This yields
\begin{align*}
2^m &\leq 
2^{L+1}\left(\frac{2e m(r\sum_{\ell=1}^L \ell m_\ell + L)}{\tilde{W} + W_L + 1}\right)^{\tilde{W} + W_L + 1}.    
\end{align*}

The last step of the proof involves applying Lemma~\ref{inequality} in~\cite{bartlett1998almost} to this inequality, which provides an upper bound for \(m\). Before doing so, we must verify that all conditions of the lemma are satisfied. In our case, \(\tilde{m}\) corresponds to \(m\), \(\kappa\) to \(L+1\), \(\tilde{w}\) to \(\tilde{W} + W_L + 1\), and \(\tilde{r}\) to \(2e(r\sum_{\ell=1}^L \ell m_\ell + L)\). Since \(r\sum_{\ell=1}^L \ell m_\ell + L > 2\), we have \(\tilde{r} > 16\). 
Moreover, we are considering the case where \( m \geq \tilde{W} + W_L + 1 \), and it is straightforward to verify that \(\tilde{W} + W_L + 1 \geq L+1 > 0\). Therefore all conditions of Lemma~\ref{inequality} in~\cite{bartlett1998almost} are indeed satisfied and we obtain
\begin{align*}
m &\leq (L+1) +2\tilde{W}\log_2 \bigg(4e\Big(r\sum_{\ell=1}^L \ell m_\ell  + L\Big) \cdot \log_2\Big(2e\Big(r\sum_{\ell =1}^L \ell m_\ell  + L\Big)\Big)\bigg).
\end{align*}

To simplify this inequality, we use that for all $a \geq 1$, $\log_2(2a\log_2 a) = \log_2(2a) + \log_2(\log_2 a) \leq 2\log_2(2a)$. 
Substituting \(a = 2e\left(r \sum_{\ell=1}^L \ell m_\ell + L\right)\), we note that \(a \leq 4e r \sum_{\ell=1}^L \ell m_\ell\) and obtain
\begin{align*}
m \leq (L+1) + 4\tilde{W}\log_2\Big(8e r\sum_{\ell =1}^L \ell m_\ell \Big),
\end{align*}
completing the proof of the theorem.
\end{proof}

\subsection{Proof of Theorem 3.5}

In this section, we provide the detailed proof of Theorem 3.5, along with the proofs for Lemmas 5.1, 5.2, 5.4, and their corresponding Corollaries 5.3 and 5.5.

The class
\begin{align}
\label{definition_GCNN_L_W_supp}
\mathcal{H}_{W,L, r} \coloneqq \left\{ \mathcal{H}(k, m_0, \ldots, m_\ell, r) \mid \ell \leq L, \, W_L \leq W \right\},
\end{align}
includes all GCNN architectures with a total number of parameters bounded by \(W\), a maximum depth of \(L\), and \(r\) representing the cardinality of the discretized group \(G^r \coloneqq \{g_1, g_2, \ldots, g_r\}\) containing the identity element \(e\).

Next, we recall that \( \mathcal{F}(m_0, \ldots, m_L) \) represents the class of fully connected feedforward ReLU networks with \(L\) layers, where \(m_i\) denotes the number of units in the \(i\)-th layer for \(i=1, \ldots, L\). The output of the last hidden layer of any neural network \(\tilde{h}_{\mathbf{w}} \in \mathcal{F}(m_0, \ldots, m_L)\), with parameters \(\mathbf{w}\), can be written as a vector of size \(m_L\), that is, \( (\tilde{h}_{\mathbf{w}}^{(1)}, \ldots, \tilde{h}_{\mathbf{w}}^{(m_L)}) \).

Finally, we define the class
\begin{align}
\label{definition_DNN_L_W_supp}
\mathcal{F}_{W,L} \coloneqq \left\{ \mathcal{F} = \mathcal{F}(m_0, \ldots, m_\ell) \mid \ell \leq L, \, W_L(\mathcal{F}) \leq W \right\},    
\end{align}
consisting of DNNs with at most \(L\) hidden layers and a total number of weights not exceeding \(W\).

\begin{lemma}
\label{Lemma 5.1}
(Lemma 5.1)  
Consider GCNNs where the G-correlation uses kernels from a one-dimensional vector space with a fixed basis given by the indicator function of the identity element \(e\). For every \(\tilde{h}_{\mathbf{w}} \in \mathcal{F}(m_0, \ldots, m_L)\), there exists a GCNN \(h_{\mathbf{w}}\) with the same number of channels in each layer, i.e., \( h_{\mathbf{w}} \in \mathcal{H}(1, m_0, \ldots, m_L, r) \), and parameters \(\mathbf{w}\), such that for any input function \(f: G^r \to \mathbb{R}^{m_0}\),
\[
\sum_{i=1}^{m_L} \sum_{j=1}^{r} \tilde{h}_{\mathbf{w}}^{(i)}(f(g_j)) = h_{\mathbf{w}}(f).
\]
\end{lemma}

\begin{proof}
Write \(\mathcal{H}:=\mathcal{H}(1, m_0, \ldots, m_L, r)\). Consider a fixed input function \(f: G^r \rightarrow \mathbb{R}^{m_0}\) and a weight vector \(\mathbf{w} \in \mathbb{R}^W\). Recall that the number of parameters in a GCNN is given by
\begin{equation}
\label{number_weights_GCNN}
W_L := \sum_{j=1}^L m_j (k m_{j-1} + 1),
\end{equation}
where \(k\) is the dimension of the kernel space. In our case \(k = 1\) and the number of parameters for a GCNN with architecture \(\mathcal{H}\) is
\begin{equation}
\label{number_weights_GCNN_final}
W_L = \sum_{j=1}^L m_j (m_{j-1} + 1).
\end{equation}
This coincides with the number of parameters in a DNN with architecture \(\mathcal{F}(m_0, \ldots, m_L)\). Consequently, the same weight vector \(\mathbf{w} \in \mathbb{R}^W\) defines both, a DNN function \(\tilde{h}_{\mathbf{w}}\) and a GCNN function \(h_{\mathbf{w}} \in \mathcal{H}\) when the input is fixed to \(f\).

We now show that the outputs of the computational units in \(\tilde{h}_{\mathbf{w}}\) and \(h_{\mathbf{w}}\) are equal when applied to \(f(g)\) and \(g\), respectively. Specifically, we aim to prove that
\[
    \tilde{h}_{\ell,i}\big(f(g_j)\big) = h_{\ell,i}(g_j),
\]
where \(\tilde{h}_{\ell,i}\) denotes a DNN computational unit in layer \(\ell\) of \(\tilde{h}_{\mathbf{w}}\), with parameters fixed to \(\mathbf{w}\), and \(h_{\ell,i}\) represents a GCNN computational unit in layer \(\ell\), with parameters fixed to \(\mathbf{w}\) and input set to \(f\). We prove this by induction on the layer \(\ell\).

The statement holds trivially for the input layer, as $\tilde{h}_{0,i}(f(g_j))=h_{0,i}(g_j)$ for any $g_j \in G^r$. Assuming it holds for all layers up to \(\ell-1\), we now prove it for layer \(\ell\).

Let \(\texttt{K}\) denote the indicator of the identity element \(e \in G^r\). By calculating the G-correlation between \(\texttt{K}\) and \(f\), we obtain \(\texttt{K} * f = f\). Combining this with the definition of the GCNN unit (see (9) in the main article) and the induction hypothesis, we have 
\begin{align*}
\tilde{h}_{\ell,i}(g_j) &:= \sigma\left(\sum_{t=1}^{m_{\ell-1}} \mathbf{w}_{t,i}^{(\ell-1)} \tilde{h}_{\ell-1,t}(g_j) - b_{i}^{(\ell)}\right)\\
&= \sigma\left(\sum_{t=1}^{m_{\ell-1}} \mathbf{w}_{t,i}^{(\ell-1)} h_{\ell-1,t}(g_j) - b_{i}^{(\ell)}\right)\qquad\qquad\qquad \text{    (induction assumption)}\\
&= \sigma\bigg(\sum_{t=1}^{m_{\ell-1}} \left( \mathbf{w}_{t,i}^{(\ell-1)} \texttt{K} * h_{\ell-1,t}\right)(g_j) - b_{i}^{(\ell)}\bigg)\quad\qquad\qquad\qquad  \text{    (property of }  \texttt{K}\text{)}\\
&= \sigma\bigg(\sum_{t=1}^{m_{\ell-1}} \left(\mathcal{K}_{ \mathbf{w}_{t,i}^{(\ell-1)}} * h_{\ell-1,t}\right)(g_j) - b_{i}^{(\ell)}\bigg)\quad  \text{(definition of learned kernel)}\\
&= h_{\ell,i}(g_j) \quad\qquad\qquad\qquad\qquad\qquad\qquad\qquad \text{(definition of GCNN unit)}.
\end{align*}
This shows that the outputs of the computational units in \(\tilde{h}_{\mathbf{w}}\) and \(h_{\mathbf{w}}\) are equal when applied to \(f(g)\) and \(g\), respectively.

Finally, the outputs of \(\tilde{h}_{\mathbf{w}} := (\tilde{h}_{\mathbf{w}}^{(1)}, \ldots, \tilde{h}_{\mathbf{w}}^{(m_L)})\) can be rewritten into the form
\begin{align*}
    \sum_{i=1}^{m_L} \sum_{j=1}^{r} \tilde{h}_{\mathbf{w}}^{(i)}(f(g_j)) 
    = \sum_{i=1}^{m_L} \sum_{j=1}^{r} \tilde{h}_{L,i}(g_j)
    = \sum_{i=1}^{m_L} \sum_{j=1}^{r} h_{L,i}(g_j) 
    = h_{\mathbf{w}}(f),
\end{align*}
concluding the proof of the lemma.
\end{proof}

Next, we prove Lemma 5.2. The key ideas and steps of the proof have already been outlined in the main article, so here we will focus on the formal statements that still needs to be established.

Recall that the indicator neural network
\begin{equation}
\label{indicator}
\mathbf{1}_{(a,b,\epsilon)}
\end{equation}
is a shallow ReLU network with four neurons in the hidden layer (see (25) in the main article). It approximates the indicator function on the interval \([a, b]\) in the sense that \( \mathbf{1}_{(a,b,\epsilon)}(x)=1 \) if \( x \in [a, b] \), and \( \mathbf{1}_{(a,b,\epsilon)}(x)=0 \) if \( x < a - \epsilon \) or \( x > b + \epsilon \).
\begin{lemma}
\label{Lemma 5.2}
[Lemma 5.2]
For \( L > 3 \) let $\mathcal{H}_{6W, L+1, r}$ be the class of GCNNs defined as in \eqref{definition_GCNN_L_W_supp} and $\mathcal{F}_{W, L}$ be the class of DNNs defined as in \eqref{definition_DNN_L_W_supp}. Then
\[
\VC(\mathcal{H}_{6W, L+1, r}) \geq \VC(\mathcal{F}_{W, L}).
\]
\end{lemma}

\begin{proof}
Let \( m \) be the VC dimension of the class of DNNs \( \mathcal{F}_{W,L} \). There are \( 2^m \) possible binary classifications for a set of \( m \) elements, subsequently denoted by \( d = 2^m \).

By definition, there exists a natural number \( m_0 \) and a set of \( m \) vectors
\begin{equation}
\label{shattering_points}
\mathcal{Y}\coloneqq \{\mathbf{y}_1, \ldots, \mathbf{y}_m\}\subset \mathbb{R}^{m_0},    
\end{equation}
that can be shattered by a subset of networks \( \tilde{\mathcal{H}} \subset \mathcal{F}_{W,L} \). Since there are no more than \( d \) distinct classifiers for \( \mathcal{Y} \), the class \( \tilde{\mathcal{H}} \) consists of at most \( d \) DNN functions.

Next, we construct a DNN architecture using \(m_0+1\) smaller DNN classes. One of these classes is \( \tilde{\mathcal{H}} \), while the remaining \( m_0 \) classes consist of "indicator" networks, as described in~\eqref{indicator}. These indicator networks ensure that the combined DNN vanishes outside a certain \( m_0 \)-dimensional hypercube. To define this hypercube, we use the set \( \mathcal{Y} \) from above.

Specifically, we choose numbers \(A > \max_{\mathbf{y} \in \mathcal{Y}} \Vert \mathbf{y} \Vert_{\infty} + 1\) and \( B > A \), and define the \( m_0 \)-dimensional hypercube
\[
\Pi \coloneqq \{\mathbf{y} = (y_1, \ldots, y_{m_0}) \mid A \leq y_i \leq B \}.
\]

To construct a DNN that vanishes on \( \Pi \), we define an approximate indicator function for \( \Pi \), using a DNN with \( m_0 \)-dimensional input \( \mathbf{y} = (y_1, \ldots, y_{m_0}) \):
\[
I: \mathbb{R}^{m_0} \rightarrow \mathbb{R}, \quad I(\mathbf{y}) \coloneqq \frac{1}{m_0} \sum_{i=1}^{m_0} \mathbf{1}_{(A, B, 0.5)}(y_i),
\]
where \( \mathbf{1}_{(A, B, 0.5)}(y_i) \) is an indicator network that approximates the indicator function for values within \( (A, B) \).

The final DNN is formed by combining functions from \( \tilde{\mathcal{H}} \) with the indicator function \( I \). Since DNNs can be summed if they have the same depth, we adjust the depth of \( I \) to match the depth of the functions from \( \tilde{\mathcal{H}} \) while ensuring that \( I \) remains constant on \( \mathcal{Y} \) and \( \Pi \). Specifically, we use the fact that for $I(\mathbf{y}) >0$, \( \sigma(I(\mathbf{y})) = I(\mathbf{y}) \) for the ReLU activation function $\sigma(x)=\max\{x,0\}$ (for any \( \mathbf{y} \) from \( \Pi \) or \( \mathcal{Y} \)). This means that by composing \( I \) with the required number of ReLU functions, we can construct a DNN that satisfies the desired properties. This construction requires at most \( L < W \) additional weights.

To complete the proof, we need to show that there are \( m \) input functions \( F_m \coloneqq \{f_1, \ldots, f_m\} \subset \{ f: G^r \rightarrow \mathbb{R}^{m_0} \} \) that can be shattered by GCNNs from \( \mathcal{H}_{6W, L+1, r} \). As the set \( F_m \), we choose functions defined by \( f_i(e) = \mathbf{y}_i \) and \( f_i(g) \in \Pi \) for \( g \in G^r \) and \( g \neq e \).

By the definition of shattering (see Definition 3.1 in the main article), to prove that \( F_m \) is shattered, it is sufficient to show that for any binary classifier \( \mathcal{C}:F_m \rightarrow \{0,1\} \), there exists a corresponding function in \( \text{sign}( \mathcal{H}_{6W, L+1, r}) \) whose values coincide with those of \( \mathcal{C} \) on \( F_m \).

Choose \( \tilde{h} \in \tilde{\mathcal{H}} \) such that for some \( b \in \mathbb{R} \), \( \text{sign}(\tilde{h}(\mathbf{y}_i) - b) = \mathcal{C}(f_i) \)
for \( i = 1, \ldots, m \).

Since \( \Pi \) is compact, we define
\[
T \coloneqq \max_{\mathbf{y} \in \Pi} |\tilde{h}(\mathbf{y})|.
\]
The final DNN \( \tilde{h}_{\mathcal{C}} \) adjusts \( \tilde{h} \) such that it vanishes on \( \Pi \) but coincides with \( \text{sign}(\tilde{h} - b) \) on \( \mathcal{Y} \),
\[
\tilde{h}_{\mathcal{C}} \coloneqq \sigma\big( \tilde{h} - (T - b)I - b \big),
\]
with $\sigma(x)=\max\{x,0\}$

Thus, for any \( f_i \in F_m \),
\[
\mathrm{sign}\left( \sum_{j=1}^{r} \tilde{h}_{\mathcal{C}}(f_i(g_j)) \right) = \mathrm{sign}(\tilde{h}(\mathbf{y}_i) - b) = \mathcal{C}(f_i).
\]
By Lemma~\ref{Lemma 5.1}, we can define a GCNN \( h_{\mathcal{C}} \) such that \( h_{\mathcal{C}}(f) = \sum_{s=1}^{r} \tilde{h}_{\mathcal{C}}\big(f(g_s)\big) \) for any \( f \in F_m \). This implies that \( \text{sign}(h_\mathcal{C}(f_i - b)) = \mathcal{C}(f_i) \) for any \( f_i \in F_m \).

As the number of weights in \( \tilde{h}_{\mathcal{C}} \) is \( W + L + 4m_0 < 6W \), this shows that our GCNN is in the class \( \mathcal{H}_{6W, L+1, r} \), completing the proof of the lemma.
\end{proof}

\begin{corollary}
\label{Corollary 5.3}
[Corollary 5.3]
In the setting of Lemma 5.2, if the number of layers \( L > 3 \) and \( L \leq  W^{0.99} \), then there exists a constant \( c \) such that 
\[
\VC(\mathcal{H}_{W, L,r}) \geq c \cdot \VC(\mathcal{F}_{W, L}).
\]
\end{corollary}    

\begin{proof}
From Equation (2) in \cite{bartlett2019nearly} , we know that for the class of fully connected neural networks $\mathcal{F}_{W,L}$ with $L$ layers and at most $W$ overall parameters, there exist constants \(c_0\) and \(C_0\) such that
\begin{align}
\label{VCbound}
 c_0 \cdot W L \log\left(\frac{W}{L}\right) \leq \VC(\mathcal{F}_{W, L}) \leq C_0 \cdot W L \log W.  
\end{align}
Moreover, by Lemma~\ref{Lemma 5.2}, we have
\[
\VC(\mathcal{H}_{6W', L'+1, r}) \geq \VC(\mathcal{F}_{W', L'}).
\]
By choosing \( W = 6W' \) and \( L = L'+1 \), this shows that
\[
\VC(\mathcal{H}_{W, L, r}) \geq \VC(\mathcal{F}_{\lfloor\frac{1}{6}W\rfloor, L-1}).
\]

To obtain the statement in the lemma, we combine this bound with the left inequality in \eqref{VCbound}, leading to
\[
\VC(\mathcal{H}_{W, L, r}) \geq \VC(\mathcal{F}_{\lfloor\frac{1}{6}W\rfloor, L-1}) \geq c_0 \cdot \left(\frac{1}{6}W-1\right)(L-1) \log\left(\frac{\frac{1}{6}W-1}{L-1}\right).
\]
For some constant \(c_1 > 0\), the right-hand side of this inequality is bounded from below by
\[
c_1 \cdot W L \log W.
\]
By using the right inequality in \eqref{VCbound}, this can be further bounded,
\[
c_1 \cdot W L \log W \geq c \cdot \VC(\mathcal{F}_{W, L}),
\]
showing the assertion.
\end{proof}

Next, we provide the proof for the second part of Theorem 3.5, which states that for some universal constant \( c > 0 \), the VC dimension \( \VC(\mathcal{H}_{W,L,r}) \) is bounded by \( c \cdot W \log_2(\resolution) \). As mentioned in the main article, the first step of the proof is Lemma 5.4.

\begin{lemma}
\label{Lemma 5.4}
[Lemma 5.4]
Let \(\mathcal{H}_{4, L, r} \) be the class of GCNNs defined in \eqref{definition_GCNN_L_W_supp}. Then
\begin{align*}
    VC(\mathcal{H}_{4, L, r}) \geq \lfloor \log_2{r} \rfloor.
\end{align*}
 Moreover, for any two numbers \( A < B \), there exists a finite subclass of GCNNs \( \mathcal{H} \subset \mathcal{H}_{4, L, r} \) 
that shatters a set of \(\lfloor log_2 r\rfloor\) input functions 
\[
F_m:=\{ f_i : G^r \rightarrow [A, B] \mid i = 1, \ldots, \lfloor\log_2 r\rfloor \},
\]
and
outputs zero for any input function \( f : G^r \rightarrow \mathbb{R} \setminus [A, B] \).
\end{lemma}

\begin{proof}
To simplify the notation, let \( m \coloneqq \lfloor \log_2 r \rfloor \). It will be enough to show that a subclass of GCNNs \( \mathcal{H} \subset \mathcal{H}_{4, L, r} \) shatters the set of $m$ input functions as this immediately implies that
\begin{align*}
   VC(\mathcal{H}_{4, L, r}) \geq \lfloor \log_2{r} \rfloor.
\end{align*}

The proof involves selecting \( d:=2^m \) distinct points in the interval \([A, B]\) and defining "indicator" neural networks of the form~\eqref{indicator} that output 1 at exactly one of these points. By adjusting the parameters of these networks, we can control the intervals of our indicator networks and ensure that each network outputs $1$ at the desired point.

Specifically, define \( \delta \coloneqq \frac{B - A}{2(d + 2)} \) and select the \( d \) points
\[
\mathcal{Y} \coloneqq \{ y_i \coloneqq A + i\delta \mid i = 1, \ldots, d \}.
\]
The input functions \( F_m \) are  now chosen from \( \{f : G^r \rightarrow \mathcal{Y} \cup \{B - \delta\}\} \).

There are \( d =2^m \) different binary classifiers for the set of \( m \) elements. Each binary classifier is defined by the elements for which it outputs \( 1 \), and we can index these classifiers by the subsets of \( \{1, 2, \dots, m\} \), denoted by \( S_1, \dots, S_d \). In our construction, each \( y_i \in \mathcal{Y} \) corresponds to the binary classifier determined by \( S_i \). More formally, the set of $m$ input functions  \( F_m \coloneqq \{ f_1, \ldots, f_m \} \) is defined by
\[
f_j(g_i) :=
\begin{cases}
    y_i, & \text{if } j \in S_i, \\
    B - \delta, & \text{otherwise}.
\end{cases}
\]

Next, we define the finite subclass in $\mathcal{H}_{4,L,r}$ that shatters \( F_m \) and outputs zero for any function $f:G^r\rightarrow \mathbb{R}\setminus [A, B]$.

 By the definition of shattering (Definition 3.1 in the main article), for any binary classifier \( \mathcal{C}: F_m \rightarrow \{-1,1\} \), we need to find a function in \( \text{sign}(\mathcal{H}_{4,L,r}) \) matching \( \mathcal{C} \) on \( F_m \).

For any classifier \( \mathcal{C}: F_m \rightarrow \{-1,1\} \) we can find a subset $S\subseteq \{1,\ldots,m\}$ such that 
$\mathcal{C}(f_j)=1$ if $j\in S$ and $\mathcal{C}(f_j)=-1$ if $j\in S^c.$ There exists an index $i^*$ such that $S=S_{i^*}.$ Using Lemma~\ref{Lemma 5.1}, one can construct a GCNN \( h_{i^*}\in  \mathcal{H}_{4,L,r}\)  that matches \( \mathcal{C} \) on \( F_m \). Indeed, for any \( f_j \in F_m \),
\begin{align}
h_{i^*}(f_j) := \sum_{s=1}^r \mathbf{1}_{(y_{i^*} - \frac{\delta}{2}, y_{i^*} + \frac{\delta}{2}, \frac{\delta}{2})}\big(f_j(g_s)\big) =
\begin{cases}
    1, & \text{if } j \in S_{i^*}, \\
    0, & \text{otherwise}.
\end{cases}
    \label{eq.9fw}
\end{align}
Thus, \( \text{sign}(h_{i^*}(f)-0.5) = \mathcal{C}(f) \) for all \( f \in F_m \).
As an 'indicator' neural network $\mathbf{1}_{(y_{i^*} - \frac{\delta}{2}, y_{i^*} + \frac{\delta}{2}, \frac{\delta}{2})}$ has only 4 parameters and 2 layers, it is in $\mathcal{H}_{4,L,r}$.

Moreover, for any $i=1,\ldots, d$ and any \( x \in \mathbb{R} \setminus [A, B] \), \( \mathbf{1}_{(y_i - \frac{\delta}{2}, y_i + \frac{\delta}{2}, \frac{\delta}{2})}(x) = 0 \). Arguing as for \eqref{eq.9fw}, \( h_{i^*}(f) = 0 \) for any \( f : G^r \rightarrow \mathbb{R} \setminus [A, B] \).

That means that the class $\mathcal{H}\coloneqq\{h_{1},\ldots, h_{d}\}$ shatters input functions $F_m$ and outputs 0 on the subset $\{f:G^r\rightarrow \mathbb{R} \setminus [A, B]\}$.
This completes the proof.
\end{proof}

\begin{corollary}
\label{Corollary 5.5}
[Corollary 5.5]
The VC dimension of the class \( \mathcal{H}_{4W,L, r} \), consisting of GCNNs with \( 4W \) weights, \( L \) layers, and resolution $r$ satisfies the inequality 
\[
\VC\big(\mathcal{H}_{4W,L, r}\big) \geq W \lfloor \log_2{r} \rfloor.
\]
\end{corollary}

\begin{proof}
To simplify notation, let \( m \coloneqq \lfloor \log_2 r \rfloor \).

We prove this corollary by defining \( W \) disjoint intervals \( [A_1, B_1], \ldots, [A_W, B_W] \), where \( A_i \coloneqq (m + 3)i \) and \( B_i \coloneqq (m + 2)i \). For different $i \in \{1, \dots, W\}$ the set of input functions \( \mathcal{F}_i \coloneqq \{ f : G^r \to [A_i, B_i] \} \) is disjoint since the values of the intervals do not overlap.

By Lemma~\ref{Lemma 5.4}, for each \( i = 1, \ldots, W \), we can find a class of GCNNs \( \mathcal{H}_i \subset \mathcal{H}_{4, L, r} \) that shatters a set of \( m \) input functions \( F_{m,i} \subset \mathcal{F}_i \) and outputs \( 0 \) on any other set \( F_{m,j} \), where \( j \neq i \).

Next we show that the class of GCNNs \( \mathcal{H} \coloneqq \mathcal{H}_1 \oplus \mathcal{H}_2 \oplus \cdots \oplus \mathcal{H}_W \subset \mathcal{H}_{4W, L, r} \) shatters the set \( F_{Wm} \coloneqq \bigsqcup_{i=1}^{W} F_{m,i} \). This will prove the corollary.

By the definition of shattering,  we need to find for any binary classifier \( \mathcal{C} : F_{Wm} \to \{0,1\} \), a function in \( \text{sign}(\mathcal{H}) \) that matches \( \mathcal{C} \) on \( F_{Wm} \).

For \( i = 1, \ldots, W \), let \( \mathcal{C}_i \coloneqq \mathcal{C} \mid_{F_{m,i}} \) be the restriction of \( \mathcal{C} \) to \( F_{m,i} \). As $\mathcal{H}_i$ shatters $F_{m,i}$, we can choose a GCNN \( h_{\mathcal{C}_i} \in \mathcal{H}_i \) such that its values match those of \( \mathcal{C}_i \) on \( F_{m,i} \).

Next, we show that the values of the GCNN \( h_{\mathcal{C}} \coloneqq \sum_{i=1}^{W} h_{\mathcal{C}_i} \) match \( \mathcal{C} \) on \( F_{Wm} \). Let \( f \) be any input function from \( F_{Wm} \), say \( f \in F_{q} \). For any \( i \neq q \), it holds that \( h_{\mathcal{C}_i}(f) = 0 \) since \( h_{\mathcal{C}_i} \in \mathcal{H}_i \). Thus, 
\[
\sum_{i=1}^{W} h_{\mathcal{C}_i}(f) = h_{\mathcal{C}_q}(f).
\]
Since \( h_{\mathcal{C}_q}(f) = \mathcal{C}(f) \) by the choice of \( h_{\mathcal{C}_q} \), it follows that \( h_{\mathcal{C}}(f) = \mathcal{C}(f) \) for any \( f \in F_{Wm} \).

This shows that the class \( \mathcal{H} \) of GCNNs shatters \( F_{Wm} \), proving the corollary.
\end{proof}

\bibliographystyle{apalike}
\bibliography{references}

\end{document}